\documentclass[10pt,letterpaper]{article}

\usepackage{times}
\usepackage{epsfig}
\usepackage{graphicx}
\usepackage{amsmath}
\usepackage{amssymb}
\usepackage{bm}
\usepackage{algorithmic}
\usepackage{algorithm}
\usepackage{amsthm}
\usepackage{placeins}

\newtheorem{theorem}{Theorem}[section]

\newtheorem{proposition}[theorem]{Proposition}

\usepackage[pagebackref=true,breaklinks=true,letterpaper=true,colorlinks,bookmarks=false]{hyperref}

\begin{document}

\title{A Kinematic Chain Space for Monocular Motion Capture}

\author{Bastian Wandt, Hanno Ackermann, Bodo Rosenhahn\\
Leibniz University Hannover\\
{\tt\small \{wandt,ackermann,rosenhahn\}@tnt.uni-hannover.de}
}

\maketitle

\begin{abstract}
This paper deals with motion capture of kinematic chains (e.g. human skeletons) from monocular image sequences taken by uncalibrated cameras. We present a method based on projecting an observation into a \textnormal{kinematic chain space (KCS)}. An optimization of the nuclear norm is proposed that implicitly enforces structural properties of the kinematic chain. Unlike other approaches our method does not require specific camera or object motion and is not relying on training data or previously determined constraints such as particular body lengths. The proposed algorithm is able to reconstruct scenes with limited camera motion and previously unseen motions. It is not only applicable to human skeletons but also to other kinematic chains for instance animals or industrial robots. We achieve state-of-the-art results on different benchmark data bases and real world scenes.

\end{abstract}

\section{Introduction}
Monocular human motion capture is an important and large part of recent research. Its applications range from surveillance, animation, robotics to medical research. While there exist a large number of commercial motion capture systems, monocular 3D reconstruction of human motion plays an important role where complex hardware arrangements are not feasible or too costly.

\begin{figure}[t]
	\begin{center}
		\includegraphics[width=0.9\linewidth]{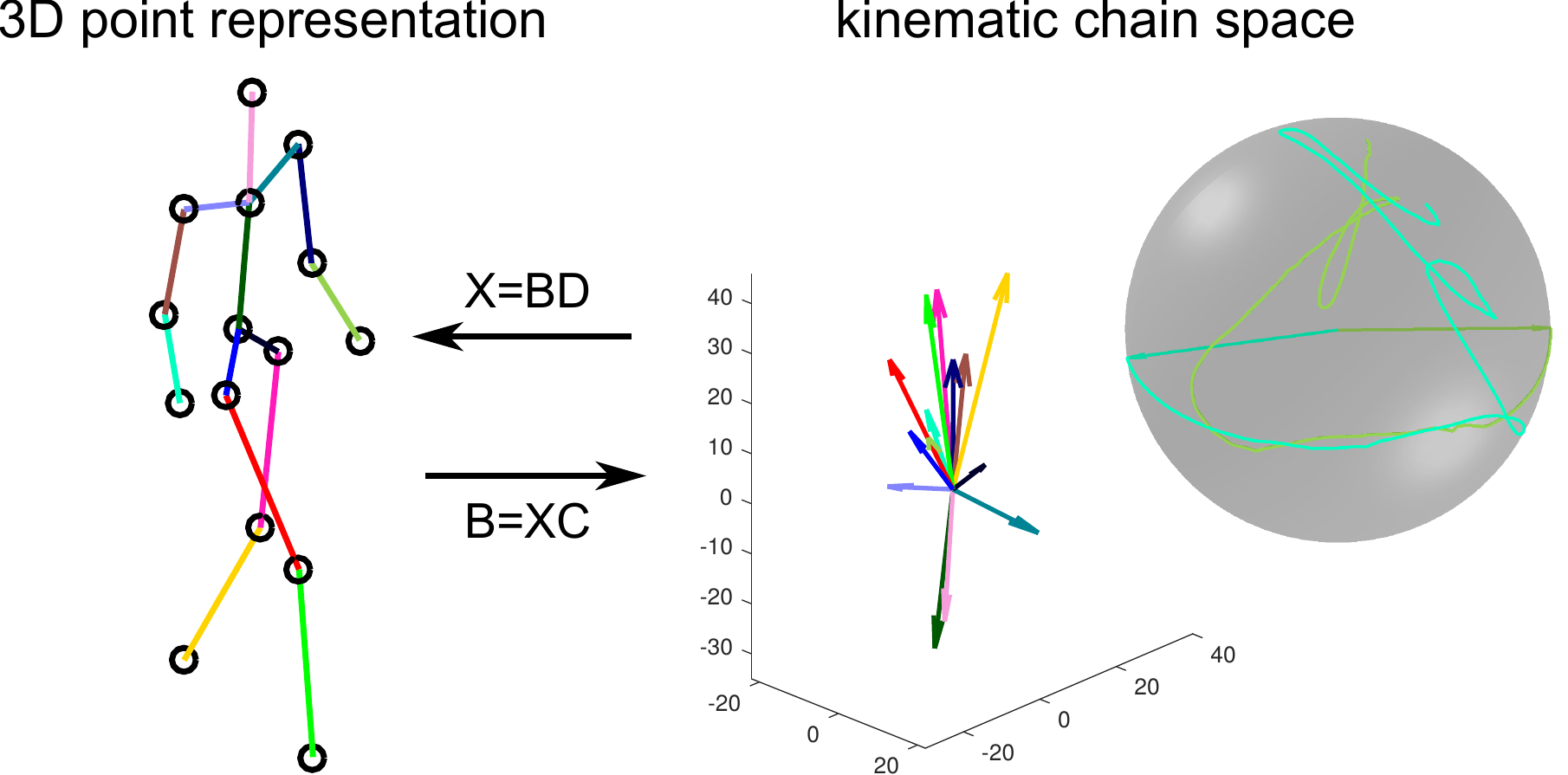}
	\end{center}
	\caption{Mapping from a 3D point representation to the kinematic chain space. The vectors in the KCS equal to directional vectors in the 3D point representation. The sphere shows the trajectories of left and right lower arm in KCS. Since both bones have the same length their trajectories lie on the same sphere.}
	\label{fig_teaser}
\end{figure}

Recent approaches to the non-rigid structure from motion problem \cite{akhter2011trajectory,gotardo12,li12,Rehan2014} achieve good results for laboratory settings. They are designed to work with tracked 2D points from arbitrary 3D point clouds. To resolve the duality of camera and point motion they require  sufficient camera motion in the observed sequence. On the other hand in many applications (e.g. human motion capture, animal tracking or robotics) properties of the tracked objects are known. Exploiting known structural properties for non-rigid structure from motion problems is rarely considered e.g. by using example based modeling as in \cite{ChenC09} or human bone lengths constancy in \cite{Wandt2016}. Recently, linear subspace training approaches have been proposed \cite{Ramakrishna12,Wang2014,Akhter2015,ZhouConvexRelax2015,Wandt2016} and appear to efficiently represent human motion, even for 3D reconstruction from single images. However, they require extensive training on known motions which restricts them to reconstructions of the same motion category. Training based approaches cannot cover individual subtleties in the motion sufficiently well.

This paper closes the gap between non-rigid structure from motion and subspace-based human modeling. Similar to other approaches which depend on the seminal work of Bregler et al. \cite{Bregler2000}, we decompose an observation matrix in three matrices corresponding to camera motion, transformation and basis shapes. Unlike other works that find a transformation which enforces properties of the camera matrices, we develop an algorithm that optimizes the transformation with respect to structural properties of the observed object. This reduces the amount of camera motion necessary for a good reconstruction. We experimentally found that even sequences without camera motion can be reconstructed. Unlike other works in the field of human modeling we propose to first project the observations in a \textit{kinematic chain space (KCS)} before optimizing a reprojection error with respect to our kinematic model. Fig.~\ref{fig_teaser} shows the mapping between the KCS and the representation based on 2D or 3D feature points. It is done by multiplication with matrices which implicitly encode a kinematic chain (cf. Sec.~\ref{sec_bone_space}). This representation enables us to derive a nuclear norm optimization problem which can be solved very efficiently. Imposing a low rank constraint on a Gram matrix has shown to improve 3D reconstructions \cite{li12}. However, the method of \cite{li12} is only based on rotational constraints for the camera matrices. So the necessary amount of camera motion is still large. Our optimization in the KCS allows for the construction of a geometric constraint based on the topology of the underlying kinematic chain. So the required amount of camera motion is much lower.

We evaluate our method on different standard data bases (CMU MoCap \cite{cmumocap}, KTH \cite{kazemi2013multi}, HumanEva \cite{humaneva}) as well as on our own data bases qualitatively and quantitatively. The proposed algorithm achieves state-of-the-art results and can handle problems like motion transfers and unseen motion. Due to the noise robustness of our method we can apply a CNN-based joint labeling algorithm \cite{Newell2016} for RGB images as input data which allows us to directly reconstruct human poses from unlabeled videos. Although this method is developed for human motion capture it is applicable to other kinematic chains such as animals or industrial robots as shown in the experiments in Sec.~\ref{sec_eval_kin_chains}.

Summarizing, our contributions are:
\begin{itemize}
	\item We propose a method for 3D reconstruction of kinematic chains from monocular image sequences.
	\item An objective function based on structural properties of kinematic chains is derived that not only imposes a low-rank assumption on the shape basis but also has a physical interpretation.
	\item We propose using a nuclear norm optimization in a \textit{kinematic chain space}.
	\item In contrast to other works our method is not limited to previously learned motion patterns and does not use strong anthropometric constraints such a-priorly determined bone lengths.
\end{itemize}

\section{Related Work}
The idea of decomposing a set of 2D points tracked over a sequence into matrices whose entries are identified with the parameters of shape and motion was first proposed by Tomasi and Kanade \cite{tomasi92}. A generalization of this algorithm to deforming shapes was proposed by Bregler et al. \cite{Bregler2000}. They assume that the observation matrix can be factorized in two matrices representing camera motion and multiple basis shapes. After an initial decomposition is found by applying an SVD to the observation matrix they compute a transformation matrix by enforcing camera constraints. Xiao et al. \cite{Xiao2004} showed that the basis shapes of \cite{Bregler2000} are ambiguous. They solved this ambiguity by employing basis constraints on them. As shown by Akther et al. \cite{akhter2011trajectory} these basis constraints are still not sufficient to resolve the ambiguity. Therefore, they proposed to use an object independent trajectory basis. Torresani et al. \cite{Torresani03,torresani08,Torresani01} proposed to use different priors on the transformation matrix such as additional rank constraints and gaussian priors. Gotardo and Martinez \cite{gotardo11} built on the idea of \cite{akhter2011trajectory} by applying the DCT representation to enforce a smooth 3D shape trajectory. Parallel to this work they proposed a solution that uses the kernel trick to also model nonlinear deformations \cite{gotardo112} which cannot be represented by a linear combination of basis shapes. Hamsici et al. \cite{gotardo12} combine the approaches of \cite{gotardo11} and \cite{gotardo112} by also assuming a smooth shape trajectory and apply the kernel trick to learn a mapping between the 3D shape and the 2D input data. Dai et al. \cite{li12} do not assume any prior knowledge about the scene such as smoothly moving cameras or points. Instead, they minimize the trace norm of the transformation matrix to impose a sparsity constraint. Since all these methods shall work for arbitrary non-rigid 3D objects, none of them utilizes knowledge about the underlying kinematic structure. Rehan et al. \cite{Rehan2014} were the first to define a temporary rigidity of reconstructed structures by factorizing only a small number consecutive frames. Therefore, they are able to model rigid structures to a small degree but due to their sliding window assumption the method is even more restricted to scenes with sufficient camera motion. 

Several works consider the special case of 3D reconstruction of human motion from monocular images. A common approach is to previously learn base poses of the same motion category. They are then linearly combined for the estimation of 3D poses. To avoid implausible poses, most authors utilize properties of human skeletons to constrain a reprojection error based optimization problem. However, anthropometric priors like the sum of squared bone lengths \cite{Ramakrishna12}, known limb proportions \cite{Wang2014}, known skeleton parameters \cite{ChenC09} or previously trained joint angle constraints \cite{Akhter2015} all suffer from the fact that parameters have to be fixed or constrained. Zhou et al. \cite{ZhouConvexRelax2015} propose a convex relaxation of the commonly used reprojection error formulation to avoid the alternating optimization of camera and object pose. While many approaches try to reconstruct human poses from a single image using anthropometric priors, such constraints have rarely been used for 3D reconstruction from image sequences. Wandt et al. \cite{Wandt2016} used a temporal bone length constraint, which does not use known skeleton parameters. Zhou et al. \cite{Zhou2016} combined a deep neural network that estimates 2D landmarks with 3D reconstruction of the human pose. The restriction to a trained subset of possible human motions is the major downside of these subspace-based approaches.

In this paper we combine NR-SfM and human pose modeling without requiring previously learned motions. By using a representation that implicitly models the kinematic chain of a human skeleton our algorithm is capable to reconstruct unknown motion from labeled image sequences.

\section{Estimating Camera and Shape}
\label{sec_cam_shape}
The $i$-th joint of a kinematic chain is defined by a vector $\bm{x}_i \in \mathbb{R}^3$ containing the $x$,$y$,$z$-coordinates of the location of this joint. By concatenating $j$ joint vectors we build a matrix representing the pose $\bm{X}$ of the kinematic chain
\begin{equation}
\label{eqn_pose}
\bm{X}=(\bm{x}_1, \bm{x}_2, \cdots, \bm{x}_j)
.
\end{equation}
The pose $\bm{X}_k$ in frame $k$ can be projected into the image plane by
\begin{equation}
	\bm{X}'_k=\bm{P}_k \bm{X}_k,
\end{equation}
where $\bm{P}_k$ is the projection matrix corresponding to a weak perspective camera. For a sequence of $f$ frames, the pose matrices are stacked such that $\bm{W}=(\bm{X}'_1, \bm{X}'_2, \dots , \bm{X}'_f)^T$ and $\hat{\bm{X}}=(\bm{X}_{1}, \bm{X}_{2}, \dots , \bm{X}_{f})^T$. This implies
\begin{equation}
\bm{W}=\bm{P} \hat{\bm{X}},
\end{equation}
where $\bm{P}$ is a block diagonal matrix containing the camera matrices $\bm{P}_{1,\dots,f}$ for the corresponding frame. In contrast to methods based on \cite{Bregler2000} we do not center the data at the mean. Instead, after an initial camera estimation we subtract a mean pose $\bm{X}_0$ from the measurement matrix by
\begin{equation}
\hat{\bm{W}}=\bm{W} - \bm{P} \hat{\bm{X}}_0
.
\end{equation}
$\hat{\bm{X}}_0$ is a matrix of stacked mean poses similar to the construction of $\bm{W}$. It is used for initializing the algorithm and will be deformed during the optimization. The matrix $\bm{X}_0$ does not need to be the mean pose of the sequence or conform to anthropometric measures of the observed object. In our experiments we use a mean pose from various motions of the CMU data set. Different motions are reconstructed from the same mean pose. Following the approach of Bregler et al. \cite{Bregler2000} we decompose $\hat{\bm{W}}$ by Singular Value Decomposition to obtain a rank-$3K$ pose basis $\bm{Q} \in \mathbb{R}^{3K\times j}$. While \cite{Bregler2000} and similar works are optimizing a transformation matrix with respect to orthogonality constraints for camera matrices, we optimize the transformation matrix with respect to constraints based on a physical interpretation of the underlying structure. With $\bm{A}$ as transformation matrix for the pose basis we may then write
\begin{equation}
\label{eqn_projection}
\bm{W}=\bm{P} (\hat{\bm{X}}_0 + \bm{A}\bm{Q})
.
\end{equation}
In the following sections we will present how poses can be projected into the kinematic chain space (Sec.~\ref{sec_bone_space}) and how we derive an optimization problem from it (Sec.~\ref{sec_problem}). Combined with the camera estimation (Sec.~\ref{sec_camera}) an alternating algorithm is presented in Sec.~\ref{sec_algorithm}.

\subsection{Kinematic Chain Space}
\label{sec_bone_space}
To define a bone $\bm{b}_k$, a vector between the $r$-th and $t$-th joint is computed by
\begin{equation}
\label{eqn_bone_diff}
\bm{b}_k=\bm{p}_r-\bm{p}_t=\bm{X}\bm{c}
,
\end{equation}
where
\begin{equation}
\bm{c}=(0,\dots, 0, 1, 0, \dots, 0 ,-1 ,0 ,\dots,0)^T
,
\end{equation}
with $1$ at position $r$ and $-1$ at position $t$. The vector $\bm{b}_k$ has the same direction and length as the corresponding bone. Similarly to Eq.~\eqref{eqn_pose}, a matrix $\bm{B} \in \mathbb{R}^{3\times b}$ can be defined containing all $b$ bones

\begin{equation}
\label{eqn_bone}
\bm{B}=(\bm{b}_1, \bm{b}_2, \dots, \bm{b}_b)
.
\end{equation}
The matrix $\bm{B}$ is calculated by 
\begin{equation}
\label{eqn_kin_space}
\bm{B}=\bm{X}\bm{C}
,
\end{equation}
where $\bm{C}\in \mathbb{R}^{j\times b}$ is built by concatenating multiple vectors $\bm{c}$. Analogously to $\bm{C}$, a matrix $\bm{D}\in \mathbb{R}^{b\times j}$ can be defined that maps $\bm{B}$ back to $\bm{X}$:
\begin{equation}
\label{eqn_map_bone_to_point}
\bm{X}=\bm{B}\bm{D}
.
\end{equation}
$\bm{D}$ is constructed similar to $\bm{C}$. Each column adds vectors in $\bm{B}$ to reconstruct the corresponding point coordinates. Note, that $\bm{C}$ and $\bm{D}$ are a direct result of the underlying kinematic chain. Therefore, the matrices $\bm{C}$ and $\bm{D}$ perform the mapping from point representation into the \textit{kinematic chain space} and vice versa.

\subsection{Trace Norm Constraint}
\label{sec_problem}
One of the main properties of human skeletons is the fact that bone lengths do not change over time.

Let 
\begin{equation}
\label{eqn_Psi}
\bm{\Psi}=\bm{B}^T\bm{B}=
\begin{pmatrix}
l_1^2 & \cdot & \cdot & \cdot \\
\cdot & l_2^2 & \cdot & \cdot \\
\cdot & \cdot & \ddots & \cdot \\
\cdot & \cdot & \cdot & l_b^2 \\
\end{pmatrix}
.
\end{equation}
be a matrix with the squared bone lengths on its diagonal. From $\bm{B}\in\mathbb{R}^{3\times b}$ follows $rank(\bm{B})=3$. Thus $\bm{\Psi}$ has rank $3$. Note, if $\bm{\Psi}$ is computed for every frame we can define a stronger constraint on $\bm{\Psi}$. As bone lengths do not change for the same person the diagonal of $\bm{\Psi}$ remains constant.
\begin{proposition}
	\label{prop_bone_lengths}
	The nuclear norm of $\bm{B}$ is invariant for any bone configuration of the same person.
\end{proposition}

\begin{proof}
	The trace of $\bm{\Psi}$ equals the sum of squared bone lengths (Eq.~\eqref{eqn_Psi})
	\begin{equation}
	trace(\bm{\Psi})=\sum_{i=1}^b l_i^2
	.
	\end{equation}
	From the assumption that bone lengths of humans are invariant during a captured image sequence the trace of $\bm{\Psi}$ is constant. The same argument holds for $trace(\sqrt{\bm{\Psi}})$. Therefore, we have
	\begin{equation}
	\|\bm{B}\|_*=trace(\sqrt{\bm{\Psi}})=const
	.
	\end{equation}
\end{proof}
Since this constancy constraint is non-convex we will relax it to derive an easy to solve optimization problem. Using Eq.~\eqref{eqn_kin_space} we project Eq.~\eqref{eqn_projection} into the KCS which gives
\begin{equation}
	\bm{W} \bm{C}=\bm{P} (\hat{\bm{X}}_0 \bm{C} + \bm{A}\bm{Q} \bm{C})
\end{equation}
The remaining unknown is the transformation matrix $\bm{A}$. For better readability we define $\bm{B}_0=\bm{X}_0 \bm{C}$ and $\bm{S}=\bm{Q} \bm{C}$.

\begin{proposition}
	\label{prop_nn}
	The nuclear norm of the transformation matrix $\bm{A}$ for each frame has to be greater than a value $c$, which is constant for each frame.
\end{proposition}

\begin{proof}
	Let $\bm{B}=\bm{B}_1+\bm{B}_0$ be a decomposition of $\bm{B}$ into the initial bone configuration $\bm{B}_0$ and a difference to the observed pose $\bm{B}_1$. It follows that 
	\begin{equation}
	\|\bm{B}\|_*=\|\bm{B}_1+\bm{B}_0\|_*=c_1
	,
	\end{equation}
	where $c_1$ is a constant.
	The triangle inequality for matrix norms gives
	\begin{equation}
	\label{eqn_B_min}
	\|\bm{B}_1\|_* +\|\bm{B}_0\|_* \geq \|\bm{B}_1+\bm{B}_0\|_*=c_1
	.
	\end{equation}
	Since $\bm{B}_0$ is known, it follows
	\begin{equation}
	\label{eqn_b_limit}
	\| \bm{B}_1 \|_* \geq c_1 - \| \bm{B}_0\|_* = c
	,
	\end{equation}
	where $c$ is constant. $\bm{B}_1$ can be represented in the shape basis $\bm{S}$ (cf. Sec.~\ref{sec_cam_shape}) by multiplying it with the transformation matrix $\bm{A}$
	\begin{equation}
	\bm{B}_1 = \bm{AS}
	.
	\end{equation}
	Since the shape base matrix $\bm{S}$ is a unitary matrix the nuclear norm of $\bm{B}_1$ equals
	\begin{equation}
	\|\bm{B}_1\|_* = \|\bm{A}\|_*
	.
	\end{equation}
	By Eq.~\eqref{eqn_b_limit} follows that
	\begin{equation}
	\label{eqn_A_tresh}
	\| \bm{A} \|_* \geq c
	.
	\end{equation}
\end{proof}
Proposition~\ref{prop_nn} also holds for a sequence of frames. Let $\bm{\hat{A}}$ be a matrix built by stacking $\bm{A}$ for each frame and $\bm{\hat{B}}_0$ be defined similarly, we relax Eq.~\eqref{eqn_A_tresh} and obtain the final formulation for our optimization problem
\begin{equation}
\label{eqn_B_opti}
\begin{aligned}
& \underset{\bm{\hat{A}}}{\min}
& & \|\bm{\hat{A}}\|_* \\
& \text{s.t.}
& & \| \bm{W} \bm{C}-\bm{P}(\bm{\hat{A}}\bm{S}+\bm{\hat{B}}_0) \|_F=0.
\end{aligned}
\end{equation}
Eq.~\eqref{eqn_B_opti} does not only define a low rank assumption on the transformation matrix. Due to our derivation and bone representation we showed that the nuclear norm is reasonable because it has a concise physical interpretation. More intuitively the minimization of the nuclear norm will give solutions close to a mean bone configuration $\bm{B}_0$ in terms of rotation of the bones, yet does not fix the solution to a predefined bone length which allows us to reconstruct arbitrary skeletons.

Moreover, Eq.~\eqref{eqn_B_opti} is a well studied problem which can be efficiently solved by common optimization methods such as Singular Value Thresholding (SVT) \cite{CandesSVT2009}.

\subsection{Camera}
\label{sec_camera}
The objective function in Eq.~\eqref{eqn_B_opti} can also be optimized for the camera matrix $\bm{P}$. Since $\bm{P}$ is a block diagonal matrix, Eq.~\eqref{eqn_B_opti} can be solved block-wise for each frame. With $\bm{X}'_i$ and $\bm{P}_i$ corresponding to the observation and camera at frame $i$ the optimization problem can be written as
\begin{equation}
\label{eqn_opt_camera}
\min_{\bm{P}_i} \| \bm{X}'_i \bm{C} - \bm{P}_i(\bm{AS}+\bm{B}_0) \|_F
.
\end{equation}

Considering the entries in

\begin{equation}
\bm{P}_i = 
\left(\begin{array}{cccc}
p_{11} & p_{12} & p_{13}\\
p_{21} & p_{22} & p_{23}\\
\end{array} \right)
\end{equation}
we can enforce a weak perspective camera by the constraints
\begin{equation}
p_{11}^2+p_{12}^2+p_{13}^2 -(p_{21}^2+p_{22}^2+p_{23}^2) = 0 
\end{equation}
and
\begin{equation}
p_{11} p_{21} + p_{12} p_{22} + p_{13} p_{23} = 0
.
\end{equation}

\subsection{Algorithm}
\label{sec_algorithm}
In the previous sections we derived an optimization problem that can be solved for the camera matrix $\bm{P}$ and transformation matrix $\bm{A}$ respectively. As both are unknown we propose algorithm~\ref{alg_alternating} which alternatingly solves for both matrices. Initialization is done by setting all entries in the transformation matrix $\bm{A}$ to zero. Additionally, an initial bone configuration $\bm{B}_0$ is required. It has to roughly model a human skeleton but does not need to be the exact mean of the sequence. 

\begin{algorithm}
	\caption{Factorization algorithm for kinematic chains}
	\label{alg_alternating}
	\begin{algorithmic}
		\STATE \% \textbf{Input:}
		\STATE $\bm{B}_0 \gets$ initial bone configuration
		\STATE $\bm{C} \gets$ kinematic chain matrix
		\STATE $\bm{W} \gets$ observation
		\STATE $f \gets$ number of frames
		\STATE $\bm{A} \gets \bm{0}$ 
		\STATE 
		
		\WHILE{no convergence}
		\FOR{$t=1 \to f$}
		\STATE optimize $\| \bm{X}_t \bm{C} - \bm{P}_t(\bm{AS}+\bm{B}_0) \|_F$
		\STATE insert $\bm{P}_t$ in $\bm{P}$
		\ENDFOR
		\STATE perform SVT on 
		\STATE \hspace{\algorithmicindent} $\min\|\bm{\hat{A}}\|_*$ s.t. $\| \bm{W}\bm{C}-\bm{P}(\bm{\hat{A}}\bm{S}+\bm{\hat{B}}_0) \|_F=0$
		
		\ENDWHILE
		\STATE 
		\STATE \% \textbf{Output:}
		\STATE $\bm{P}$: camera matrices 
		\STATE $(\bm{\hat{A}}\bm{S}+\bm{\hat{B}}_0)\bm{D}$: 3D poses
	\end{algorithmic}
\end{algorithm}

\section{Experiments}
For the evaluation of our algorithm different benchmark data sets (CMU MoCap \cite{cmumocap}, HumanEva \cite{humaneva}, KTH \cite{kazemi2013multi}) were used. As measure for the quality of the 3D reconstructions we calculate the \textit{Mean Per Joint Position Error (MPJPE)} \cite{Ionescu14} which is defined by

\begin{equation}
e = \frac{1}{j} \sum_{i=1}^{j} \| \bm{x}_i - \hat{\bm{x}}_i \|
,
\end{equation}
where $\bm{x}_i$ and $\hat{\bm{x}}_i$ correspond to the ground truth and estimated positions of the $i$-th joint respectively. By rigidly aligning the 3D reconstruction to the ground truth we obtain the \textit{3D positioning error (3DPE)} as introduced by \cite{SimoSerraRATM12}. To compare sequences of different lengths the mean of the 3DPE over all frames is used. In the following it is referred to as \textit{3D error}.

Additional to this quantitative evaluation we perform reconstructions of different kinematic chains in Sec.~\ref{sec_eval_kin_chains} and on unlabeled image sequences in Sec.~\ref{sec_unlabeled_sequences}. All animated meshes in this section are created using SMPL \cite{SMPL2015}.

\subsection{Evaluation on Benchmark Databases}
\begin{figure*}[t]
	\begin{center}
		\includegraphics[width=0.85\linewidth]{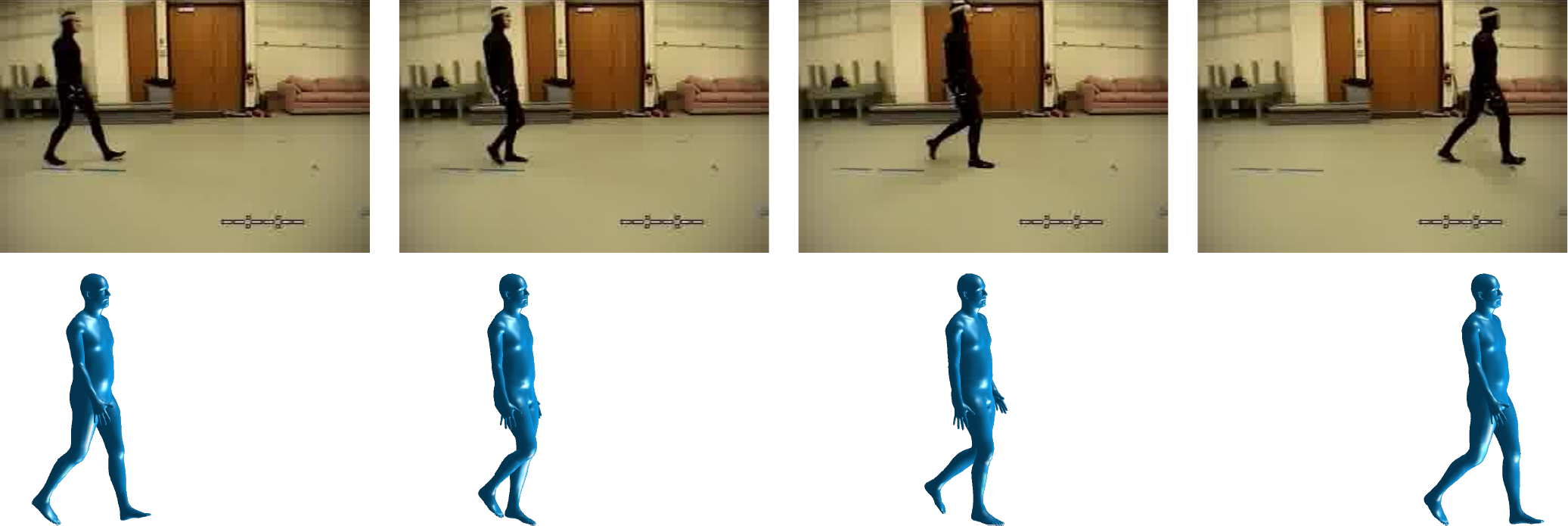}
	\end{center}
	\caption{Reconstruction of a walking motion from the CMU data base subject 35/02.}
	\label{fig_walk}
\end{figure*}

\begin{figure*}[!ht]
	\begin{center}
		\includegraphics[width=0.85\linewidth]{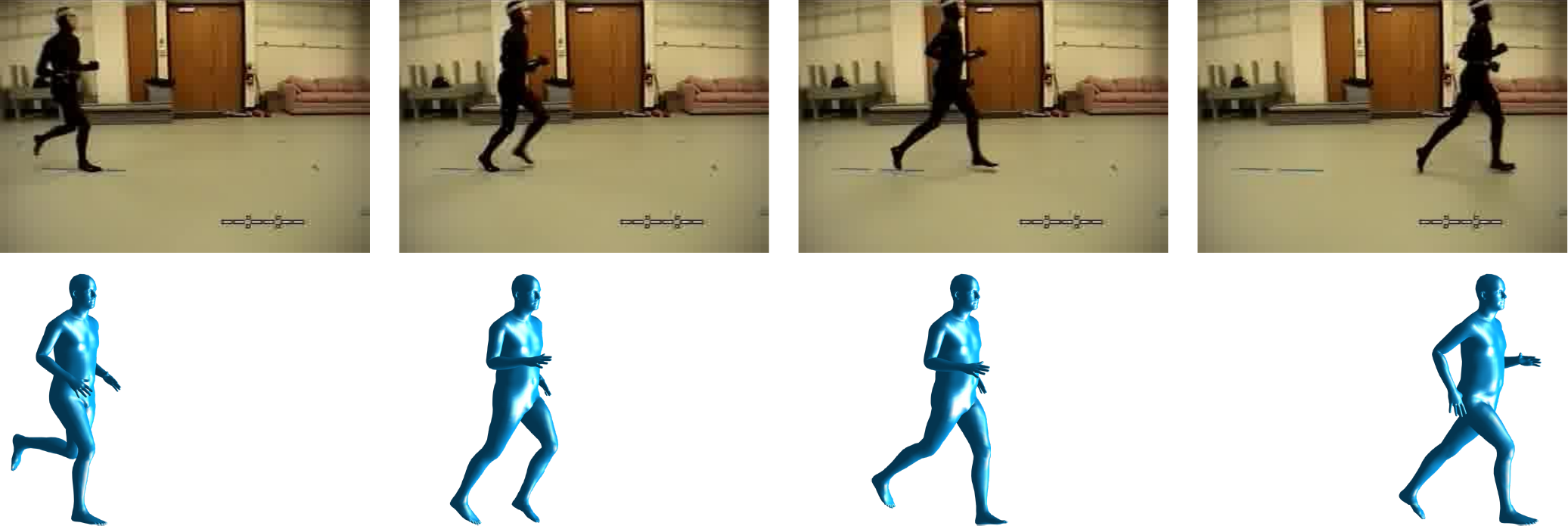}
	\end{center}
	\caption{Reconstruction of a running motion from the CMU data base subject 35/17.}
	\label{fig_run}
\end{figure*}

\begin{figure}[t]
	\begin{center}
		\includegraphics[width=0.9\linewidth]{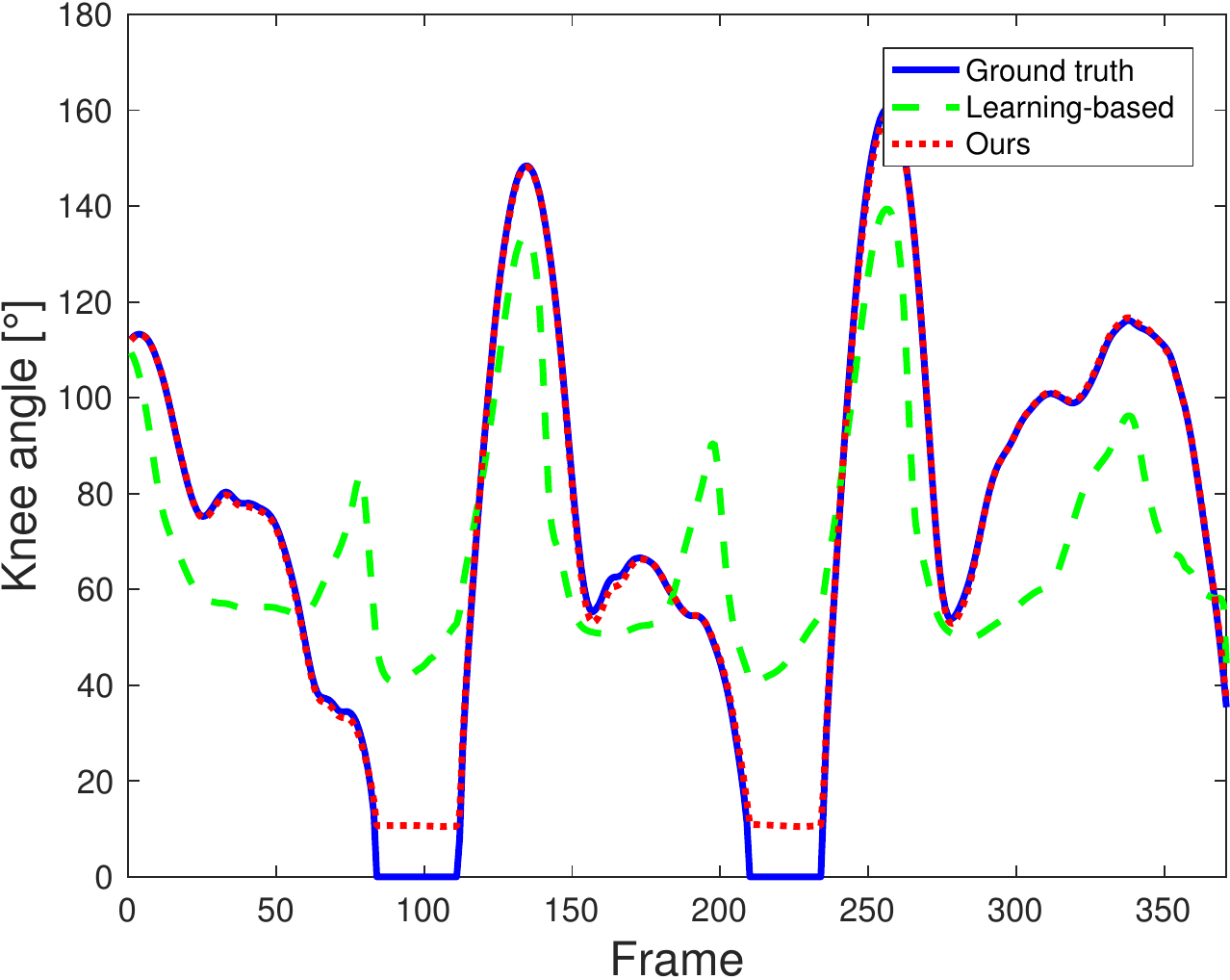}
	\end{center}
	\caption{Knee angle of reconstructions of a limping motion. The learning-based method \cite{Wandt2016} struggles to reconstruct minor differences from the motion patterns used for training whereas our learning-free approach recovers the knee angle in more detail.}
	\label{fig_baseposes}
\end{figure} 

To qualitatively show the drawbacks of learning-based approaches we reconstructed a sequence of a limping person. We use the method of \cite{Wandt2016} trained on walking patterns to reconstruct the 3D scene. Although the motions are very similar, the algorithm of \cite{Wandt2016} is not able to reconstruct the subtle motions of the limping leg. Fig.~\ref{fig_baseposes} shows the knee angle of the respective leg. The learning-based method reconstructs a periodic walking motion and cannot recover the unknown asymmetric motion. The proposed algorithm is able to recover the motion in more detail.

Since we do not require training data, we compare with other unsupervised works. For each sequence we created 20 random camera paths with low camera motion and compared our 3D reconstruction results with other state-of-the-art methods \cite{akhter2011trajectory,gotardo12}. Table~\ref{tab_eval_results} shows the 3D error in $mm$ for different sequences and data sets. For the entry \textit{walk35} we calculated the mean overall 3D errors of all 23 walking sequences from subject 35 in the CMU data base. The column \textit{jump} shows the 3D error of a single jumping sequence. \textit{KTH} means the football sequence of the KTH data set \cite{kazemi2013multi}. All these sequences are captured with limited camera motion. 
Methods like \cite{akhter2011trajectory} and \cite{gotardo12} require more camera motion and completely fail in these scenarios. Some of our reconstructions are shown in Figs.~\ref{fig_walk} and \ref{fig_run} for a walking and running motion.

\begin{table}
	\centering
\caption{3D error in $mm$ for different sequences and data sets. The column \textit{walk35} shows the mean 3D error of all sequences containing walking motion from subject 35 in the CMU data base. \textit{jump} refers to the jumping motion of subject 13/11 of the CMU data base. \textit{KTH} means the football sequence of the KTH data set \cite{kazemi2013multi}. The column \textit{HE} shows the 3D error for the HumanEva walking sequence \cite{humaneva}.}
\label{tab_eval_results}
\begin{tabular}{| l | c | c | c | c |}
\hline
 & walk35 & jump & KTH & HE \\ \hline\hline
\cite{akhter2011trajectory} & 228.68 & 210.14 & 108.91 & 106.92 \\
\cite{gotardo12} & 264.75 & 186.70 & 114.03 & 102.99 \\
Ours & \textbf{18.94} & \textbf{36.50} & \textbf{53.10} & \textbf{44.36} \\
\hline
\end{tabular}
\end{table}

\subsection{Convergence}
\begin{figure}[t]
	\begin{center}
		\includegraphics[width=0.85\linewidth]{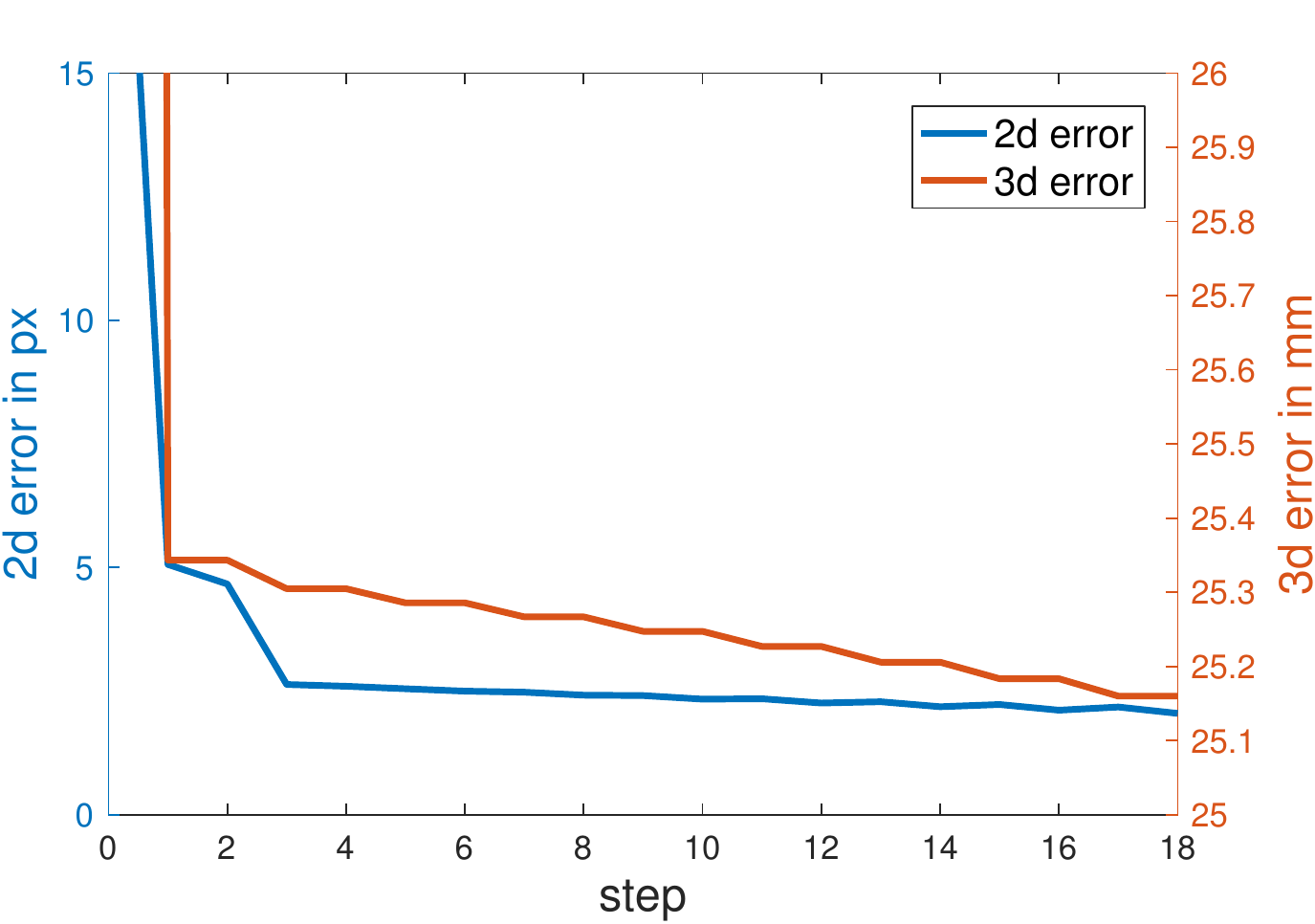}
	\end{center}
	\caption{Reprojection error and 3D error with respect to number of iterations for subject35/sequence1 from the CMU MoCap data set. Even steps refer to camera estimation while odd steps correspond to shape estimation.}
	\label{fig_convergence}
\end{figure}
\begin{figure}[t]
	\begin{center}
		\includegraphics[width=0.85\linewidth]{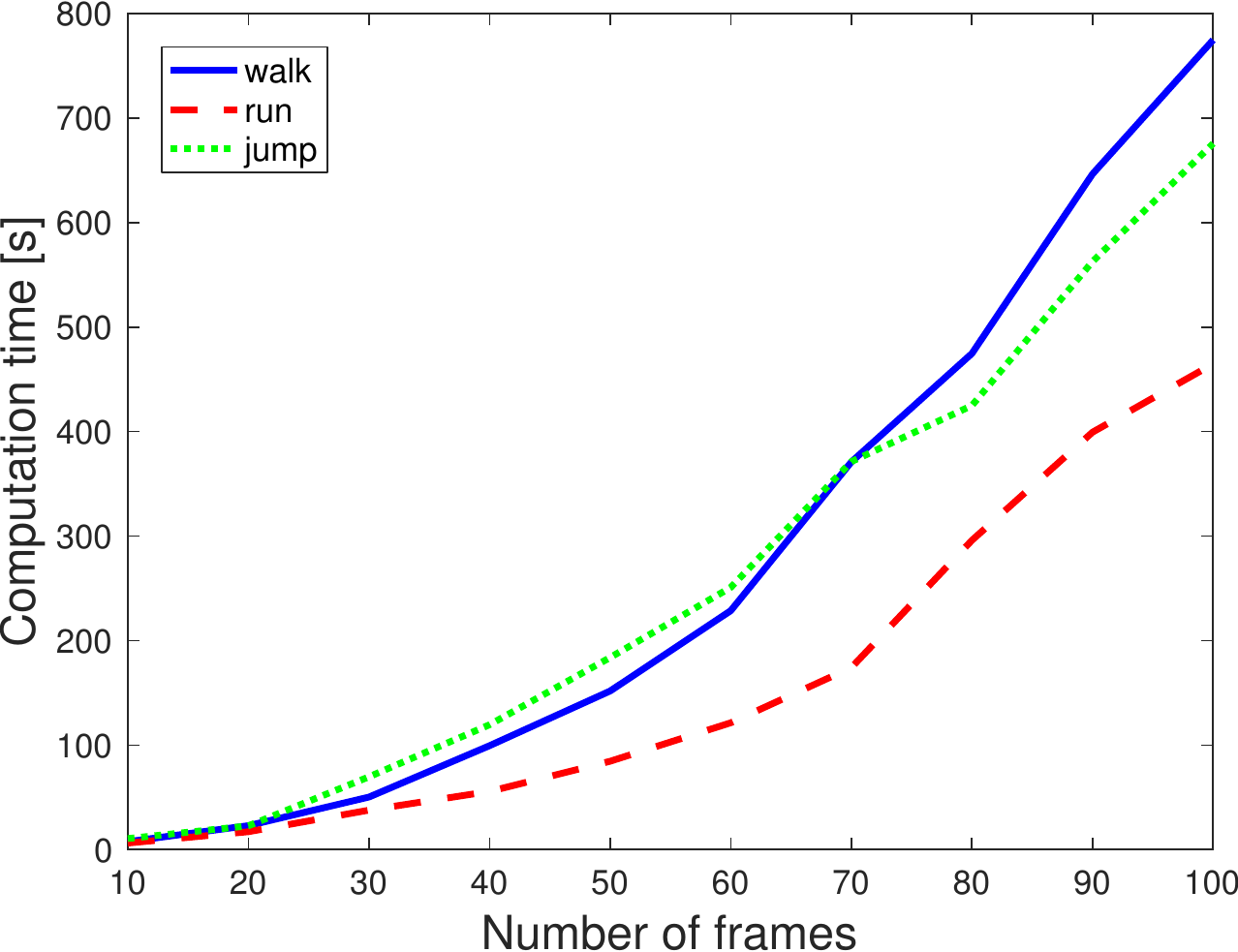}
	\end{center}
	\caption{Computation time for walking, running and jumping sequences of the CMU data set using unoptimized Matlab code. It mostly depends on the number frames and less on the observed motion.}
	\label{fig_comp_time}
\end{figure}
\begin{figure*}[t]
	\begin{center}
		\includegraphics[width=0.85\linewidth]{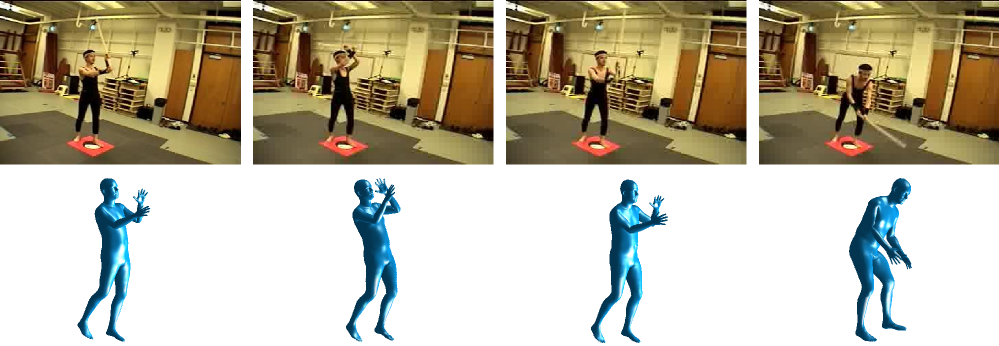}
	\end{center}
	\caption{Reconstruction of the sword play sequence of the CMU data base. The kinematic chain is extended such that the hands are rigidly connected.}
	\label{fig_sword}
\end{figure*}
We alternatingly optimize the camera matrices (Eq.~\eqref{eqn_B_opti}) and transformation matrix (Eq.~\eqref{eqn_opt_camera}). Since convergence of the algorithm cannot be guaranteed we show it by experiment. Fig.~\ref{fig_convergence} shows the convergence of the reprojection error in pixel for a sequence from the CMU MoCap data base. However, the reprojection error only shows the convergence of the proposed algorithm but cannot prove that the 3D reconstructions will improve every iteration. We additionally estimated the convergence of the 3D error in Fig.~\ref{fig_convergence}. In most cases our algorithm converges to a good minimum in less than $3$ iterations. Further iterations do not improve the visual quality and only deform the 3D reconstruction less than $1mm$. The 3D error remains constant during camera estimation which causes the \textit{steps} in the error plot.

Fig.~\ref{fig_comp_time} shows the computation time over the number of frames for three different sequences. The computation time mostly depends on the number frames and less on the observed motion. We use unoptimized Matlab code on a desktop PC for all computations. 

\subsection{Other Kinematic chains}
\label{sec_eval_kin_chains}

\begin{figure*}[t]
	\begin{center}
		\includegraphics[width=0.8\linewidth]{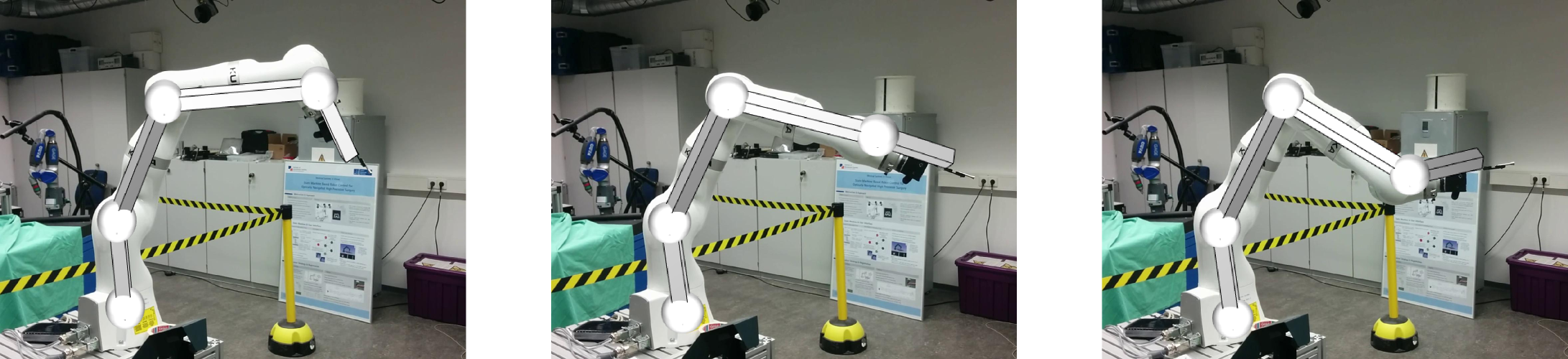}
	\end{center}
	\caption{Reconstruction of a sequence of an industrial robot moving along a path. The reconstruction is shown as an augmented overlay over the images.}
	\label{fig_robot}
\end{figure*}
\begin{figure*}[t]
	\begin{center}
		\includegraphics[width=0.8\linewidth]{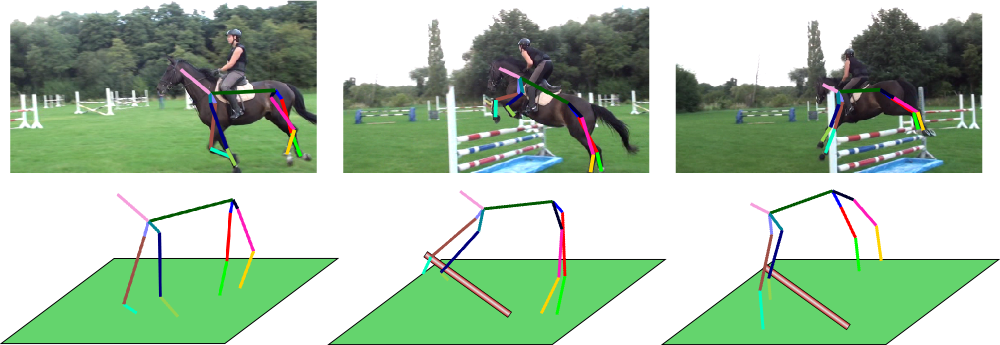}
	\end{center}
	\caption{Reconstruction of a horse riding sequence. Although we use a very rough model for the skeleton of the horse we obtain plausible reconstructions.}
	\label{fig_horse}
\end{figure*}

Although our method was developed for the reconstruction of human motion it generalizes to all kinematic chains that do not include translational joints. In this section we show reconstructions of other kinematic chains such as people holding objects, animals and industrial robots.

In situations where people hold objects with both hands the kinematic chain of the body can be extended by another rigid connection between the two hands. Fig.~\ref{fig_sword} shows the reconstruction of the sword fighting sequence of the CMU data set. By simply adding another column to the kinematic chain space matrix $\bm{C}$ (cf. Sec.~\ref{sec_bone_space}) the distance between the two hands is enforced to remain constant. The distance does not need to be known.

Fig.~\ref{fig_robot} shows a robot used for precision milling and the reconstructed 3D model as overlay. The proposed method is able to correctly reconstruct the robots motion. In Fig.~\ref{fig_horse} we reconstructed a more complex motion of a horse during show jumping. We used a simplified model of the bone structure of a horse. Also in reality the shoulder joint is not completely rigid. Despite these limitations the algorithm achieves plausible results.

\subsection{Image Sequences}
\label{sec_unlabeled_sequences}
\begin{figure*}[t]
	\begin{center}
		\includegraphics[width=0.8\linewidth]{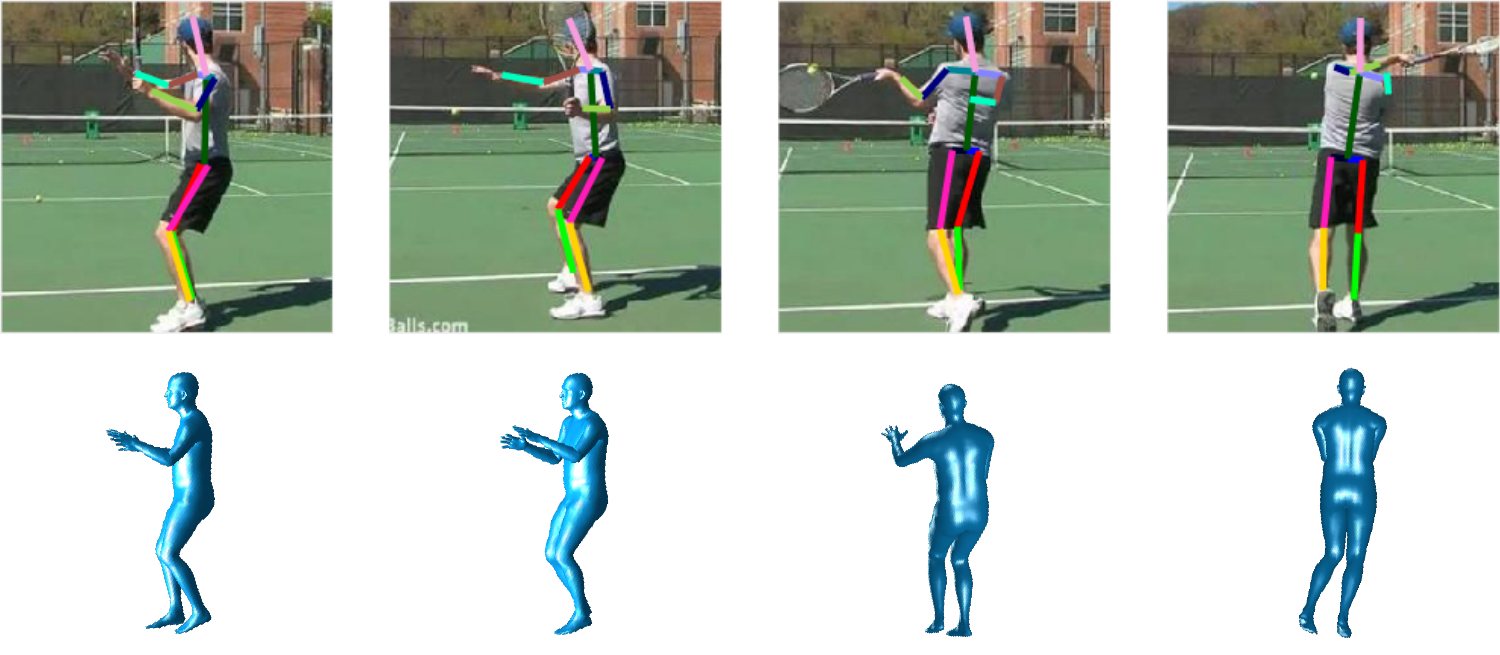}
	\end{center}
	\caption{Reconstruction of a tennis sequence automatically labeled by a CNN \cite{Newell2016}.}
	\label{fig_cnn}
\end{figure*}
The proposed method is designed to reconstruct a 3D object from labeled feature points. In the former sections this was mainly done by setting and tracking them semi-interactive. In this section we will show that our method is also able to use the noisy output of a human joint detector. We use the pre-trained CNN of Newell et al. \cite{Newell2016} to estimate the joints in a tennis sequence. Fig.~\ref{fig_cnn} shows the joints estimated by the algorithm of \cite{Newell2016} and our 3D reconstruction.

\section{Conclusion}
We developed a method for the 3D reconstruction of kinematic chains from monocular image sequences. By projecting into the kinematic chain space a constraint is derived  that is based on the assumption that bone lengths are constant. This results in the formulation of an easy to solve nuclear norm optimization problem. It allows for reconstruction of scenes with little camera motion where other non-rigid structure from motion methods fail. Our method does not rely on previous training or predefined body measures such as known limb lengths. The proposed algorithm generalizes to the reconstruction of other kinematic chains and achieves state-of-the-art results on benchmark data sets.

{\small
\bibliographystyle{ieee}
\bibliography{literature}
}

\end{document}